\documentclass[submission,copyright,creativecommons]{eptcs}
\usepackage{underscore}           % Only needed if you use pdflatex.

\usepackage{amssymb}
\usepackage{amsmath}
\usepackage{stmaryrd}
\usepackage{comment}
\usepackage[all, knot]{xy}
\usepackage{charter}
\usepackage{color}
\usepackage[english]{babel}
\usepackage{comment}
\usepackage{xcolor}
\usepackage[ruled,vlined]{algorithm2e}
\usepackage{caption}
\newenvironment{proof} {\textsc{Proof}\quad} {\hfill $\Box$\\}
\definecolor{citecolor}{rgb}{0.5,0.5,0.5}
\newcommand{\BP}{\ensuremath{\mathbf{P}}}

\newcommand{\BA}{\ensuremath{\mathbf{A}}}

\newcommand{\Act}{\ensuremath{\mathbf{A}}}
\newcommand{\Ag}{\ensuremath{\mathbf{I}}}

\newcommand{\M}{\mathcal{M}}

\newcommand{\rel}[1]{\xrightarrow{#1}}

\newcommand{\lr}[1]{\langle #1 \rangle}

\newcommand{\AxTr}{\ensuremath{\mathtt{T}}}

\newcommand{\EAL}{\textsf{EAL}}

\newcommand{\ELKh}{\textsf{ELKh}}

\newcommand{\ptime}{P{\small TIME}}

\newcommand{\DISTK}{\ensuremath{\mathtt{DISTK}}}

\newcommand{\TAUT}{\ensuremath{\mathtt{TAUT}}}

\newcommand{\MP}{\ensuremath{\mathtt{MP}}}
\newcommand{\ZIG}[1]{\ensuremath{\mathtt{PR}(#1)}}

\newcommand{\SUB}{\ensuremath{\mathtt{SUB}}}

\newcommand{\bigA}{\Act}

\newtheorem{theorem}{Theorem}
\newtheorem{definition}[theorem]{Definition}
\newtheorem{lemma}[theorem]{Lemma}
\newtheorem{proposition}[theorem]{Proposition}
\newtheorem{example}[theorem]{Example}

\newcommand{\bis}{\mathrel{\mathchoice%
		{\raisebox{.3ex}{$\,
				\underline{\makebox[.7em]{$\leftrightarrow$}}\,$}}%
		{\raisebox{.3ex}{$\,
				\underline{\makebox[.7em]{$\leftrightarrow$}}\,$}}%
		{\raisebox{.2ex}{$\,
				\underline{\makebox[.5em]{\scriptsize$\leftrightarrow$}}\,$}}%
		{\raisebox{.2ex}{$\,
				\underline{\makebox[.5em]{\scriptsize$\leftrightarrow$}}\,$}}}}

\newcommand{\Kh}{\mathsf{Kh}}

\newcommand{\K}{\mathsf{K}}

\newcommand{\V}{\mathcal{V}}

\newcommand{\lra}{\leftrightarrow}

\renewcommand{\phi}{\varphi}

\newcommand{\LRA}{\Leftrightarrow}

\DeclareSymbolFont{symbolsC}{U}{txsyc}{m}{n}
\DeclareMathSymbol{\strictif}{\mathrel}{symbolsC}{74}
\newcommand{\AxTrK}{\ensuremath{\mathtt{T}}}
\newcommand{\AxTransK}{\ensuremath{\mathtt{4}}}
\newcommand{\AxEucK}{\ensuremath{\mathtt{5}}}
\newcommand{\EQREPKh}{\ensuremath{\mathtt{MONOKh}}}

\newcommand{\AxKtoKh}{\ensuremath{\mathtt{AxKtoKh}}}
\newcommand{\AxKhtoKhK}{\ensuremath{\mathtt{AxKhtoKhK}}}
\newcommand{\AxKhtoKKh}{\ensuremath{\mathtt{AxKhtoKKh}}}

\newcommand{\AxKhKh}{\ensuremath{\mathtt{AxKhKh}}}
\newcommand{\AxKhbot} {\ensuremath{\mathtt{AxKhbot}}}
\newcommand{\NECK}{\ensuremath{\mathtt{NECK}}}

\newcommand{\cl}[1]{cl(#1)}
\newcommand{\SFC}{\Phi}
\newcommand{\ACT}{\Act}

\newcommand{\SLKhc}{\mathbb{SLKHC}}

\newcommand{\Prg}{\ensuremath{\mathtt{KBP}}}
\newcommand{\ifthel}[3]{\ensuremath{\mathtt{if}{~#1~}\mathtt{then}{~#2~}\mathtt{else}{~#3} }}

\newcommand{\Ptime}{PTIME}
\newcommand{\shortIf}[3]{#1 ? #2\!:#3}
\newcommand{\T}{\mathcal{T}}

\title{Knowing How to Plan}
\author{Yanjun Li
	\institute{College of Philosophy,\\ Nankai University, 
		Tianjin, China}
	\and
	Yanjing Wang\thanks{Corresponding author. Yanjing Wang thanks the support of NSSF grant 19BZX135.}
	\institute{
		Department of Philosophy,\\ Peking University, 
		Beijing, China
		}
}
\def\titlerunning{Knowing How to Plan}
\def\authorrunning{Yanjun Li \& Yanjing Wang}
\begin{document}
\maketitle

\begin{abstract}
Various \textit{planning-based know-how logics} have been studied in the recent literature. In this paper, we use such a logic to do \textit{know-how-based planning} via model checking. In particular, we can  handle the \textit{higher-order epistemic planning} involving know-how formulas as the goal, e.g., find a plan to make sure $p$ such that the adversary does not know how to make $p$ false in the future. We give a PTIME algorithm for the model checking problem over finite epistemic transition systems and axiomatize the logic under the assumption of perfect recall. 
\end{abstract}
% \keywords{epistemic planning, epistemic logic, logic of know-how, model checking }

\section{Introduction}

Standard Epistemic Logic (EL) mainly studies reasoning patterns of \textit{knowing that $\phi$}, despite early contributions  by Hintikka on formulating other \textit{know-wh} expressions such as \textit{knowing who} and \textit{why} using first-order and higher-order modal logic. In recent years, there is a resurgence of interest on epistemic logics of \textit{know-wh} powered by the new techniques for fragments of first-order modal logic based on the so-called \textit{bundle modalities} packing a quantifier and a normal epistemic modality together \cite{Wang2018,Wang2017,Padmanabha2018}. Within the varieties of logics of \textit{know-wh}, the logics of know-how received the most attention in AI (cf. e.g., \cite{Wang2016,Fervari2017,Naumov2018a,LiWangAIJ2021}). 

%\subsection{From planning-based know-how to know-how based planning}

Besides the inspirations from philosophy and linguistics (e.g., \cite{stanley2001knowing}), the idea of \textit{automated planning} in AI also plays an important role in the developments of various logics of know-how. The core idea is (roughly) to interpret \textit{knowing how to achieve $\phi$} as a \textit{de re} statement of knowledge: ``\textit{there is} a plan such that you \textit{know that} this plan will \textit{ensure} $\phi$''. Here, depending on the exact notion of \textit{plans} and how much we want to ``ensure'' that the plan works, there can be different semantics based on ideas from conformant planning and contingent planning in AI. However, as shown by Li and Wang \cite{LiWangAIJ2021}, there is a logic core which is independent from the exact notions of plans underlying all these notions of know-how. We can actually unify different planning-based approaches in a powerful framework.

In this paper, we show that the connection between planning and know-how is not merely one way in terms of \textbf{\textit{planning-based know-how}}, it also makes perfect sense to do \textbf{\textit{know-how-based planning}} which generalizes the notion of planning to incorporate know-how based goals to be explained in the next subsection. As  observed in \cite{LiWangAIJ2021}, the typical epistemic planning problem given explicit models can be viewed as model checking problems in our framework with the know-how modality $\Kh$ in the language. For example, $\Kh_1 (\K_2\phi \land \neg \K_3\K_2\phi)$ captures the epistemic planning problem for agent 1 to ensure that agent 2 knows that $\phi$ and keep it a secret from agent 3. Such epistemic planning can also be done in other epistemic approaches, for example, by using dynamic epistemic logic (cf. e.g., \cite{BA2011,Bolander2015,BOLANDER202010,EngesserMNT21}).  However, we can do much more with the $\Kh$ modality in hand, as explained below.

\subsection{Higher-order epistemic planning}
A distinct feature of modal logic is that we can bring some notions of the  meta-language to the object-language. This is not merely formalizing the existing meta-language concepts more precisely since the object language can open new possibilities. Consider the following multi-agent epistemic language of know-that and know-how: 
	\[\phi::=\top\mid p\mid\neg\phi\mid (\phi\land\phi)\mid\K_i\phi\mid \Kh_i\phi \]
Since we have the know-how operators $\Kh_i$ in the fully compositional logical language, the goal of planning can involve the (Boolean combinations and nestings of) know-how formulas as well, and we call such planning problems  \textbf{ \textit{higher-order epistemic planning}}, i.e., \textit{planning about planning}. By \textit{higher-order} we do not mean higher-order logic or higher-order epistemic goals in terms of nested know-that formulas.

Although in the single-agent scenario, $\Kh_1\Kh_1\phi$ is equivalent to $\Kh_1\phi$ under reasonable conditions as shown in \cite{Fervari2017,LiWangAIJ2021}, it still makes sense to do higher-order planning. For example, the planning problem to achieve $\phi$ with the future control to turn it off is expressed by $\Kh_1(\phi \land \Kh_1\neg\phi)$, which is not reducible to a formula without the nesting of $\Kh_1$.\footnote{In contrast to normal modal logic, $(\Kh_i\phi\land\Kh_i\psi)\lra \Kh_i(\phi\land \psi)$ is \textit{invalid} thus $\Kh_i\Kh_i\phi\lra \Kh_i\phi$ cannot be applied, cf. e.g., \cite{LiWangAIJ2021}.}

In a multi-agent setting, given that different agent may have different abilities, the importance of higher-order planning is self-evident. Actually, it is a characteristic human instinct to plan in such a higher-order way. A one-year-old baby girl may not know how to open a bottle but she knows how to use her parents' knowledge-how to achieve her goal.\footnote{As new fathers, both authors know this well and sometimes it can be tiring to be a fully cooperative parent. At the same time, the parents are eager to know how to let their children acquire relevant knowledge-how as early as possible.} 
%On the other hand, the parents are eager to know how to let small kids acquire relevant knowledge-how as early as possible. 
As a more concrete example about academic collaborations, it often happens that both researchers do not know how to prove a theorem independently, but one knows how to show a critical lemma which can simplify the original problem and thus enable the other to prove the final theorem using her expertise about the simplified statement. In our language, the situation can be expressed as $\neg \Kh_1 \phi \land \neg \Kh_2\phi \land \Kh_1\Kh_2\phi.$ Of course, it depends on the cooperation of agent 2 to finally ensure the goal $\phi$ of agent 1. The  nesting of know-how can go arbitrarily deep and also be interactive such as  $\Kh_1\Kh_2\Kh_1\Kh_2\phi$. Such planning based on others' knowledge-how is at the core of the arts of leadership and management in general. 

Higher-order planning also makes perfect sense in non-cooperative settings. For example, it is sometimes important to ensure not only the goal $\phi$ but also that the adversary cannot spoil it in the future. This is expressed by $\Kh_1(\phi\land \neg \Kh_2 \neg\phi)$ in our language of knowing how. It is also interesting to have mixed goals with both $\K$ and $\Kh$,  such as showing a commitment by achieving $\phi$ while letting the other know you will not be able to change it afterwards: $\Kh_1\K_2(\phi\land \neg \Kh_1\neg\phi)$. In the more interactive game-like scenarios, $\Kh_1\neg \Kh_2\neg \Kh_1 \phi$ describe a winning strategy-like plan. We will come back to the related work using \textit{Alternating-time Temporal Logic} (ATL)-like logics at the end of this introduction. Let us first see a concrete example.   
\begin{example}
Suppose a patient is experiencing some rare symptom $p$. To know the cause ($q$ or $\neg q$), Doctor 1 in the local hospital suggests the patient first take a special fMRI scan ($a$) which is only available in the local hospital at the moment, and bring the result to Doctor 2 in a hospital in another city, who is more experienced in examining the result from the scanner. Then Doctor 2 will then know the cause, depending on which different treatments ($b$ or $c$) should be performed by Doctor 2 to cure the patient.  Intuitively, Doctor 1 knows how to let Doctor 2 know how to cure the patient although neither Doctor 1 nor Doctor 2 knows how to cure it without the help of the other. The situation is depicted below as a model with both epistemic relations (dotted lines labelled by $1,2$) and action relations (solid lines labelled by $a,b,c$). Note that only $1$ can execute action $a$ and only 2 can perform $b$ or $c$. Reflexive epistemic arrows are omitted. 
$$
\xymatrix{
s_1:p\ar@{.}[d]|{1,2}\ar[r]|a& s_3:p\ar@{.}[d]^1 \ar[r]|b\ar[rd]|c& s_5:\\
s_2:p,q\ar[r]|a& s_4:p,q \ar[ru]|c\ar[r]|b& s_6: p, q \\
% \save "1,1"."2,1"!C="g1"*+[F--:<+20pt>]\frm{}
% \restore
% \save "1,1"."2,1"!C="g1"*+[F.:<+30pt>]\frm{}
% \restore
% \save "1,2"."1,2"!C="g1"*+[F--:<+20pt>]\frm{}
% \restore
% \save "2,2"."2,2"!C="g2"*+[F--:<+20pt>]\frm{}
% \restore
% \save "1,3"."1,3"!C="g1"*+[F--:<+20pt>]\frm{}
% \restore
% \save "2,3"."2,3"!C="g2"*+[F--:<+20pt>]\frm{}
% \restore
}
%\vspace{-15pt}
$$
We would like to have $\neg \Kh_1 \neg p\land \neg \Kh_2 \neg p\land \Kh_1((\K_2 q\lor\K_2\neg q)\land \Kh_2 \neg p)$ hold at $s_1$ according to the semantics to be introduced. 
\end{example}\label{ex.symptom}
In this paper, to facilitate such higher-order planning using model checking, we extend the single-agent framework in \cite{LiWangAIJ2021} to a genuine multi-agent setting and study the model checking complexity in details. Technical and computational complications arise due to the introduction of multiple agents with different abilities and knowledge to base their own plans. To stay as close as the intended application of automated planning, we will restrict ourselves to finite models with perfect recall.\footnote{cf. \cite{LiWangAIJ2021} for the discussions on the implicit assumption of perfect recall in epistemic planning and the technical difficulty of capturing it by axioms.}
%\noteYW{Highlight the differences and new method.}
%\noteLYJ{
%We also proposed a proof system of this logic. 
% Compared to the proof systems presented in \cite{LiWangAIJ2021}, due to the fact that we focus on models with perfect recall in this paper, the proof system here is stronger then that in \cite{LiWangAIJ2021}. 
The model checking algorithm is based on a correspondence between execution trees in the models and the syntactically specified  knowledge-based plans. It turns out that this correspondence result also enables us to significantly simplify the proofs of the soundness and completeness of the proof system compared with the method in \cite{Fervari2017,LiWangAIJ2021}.  % plays an important role in proofs of the soundness and completeness of the proof system. It enables us to significantly simplify the proofs 
%}
\medskip

We summarize our contributions below.
\begin{itemize}
    \item A multi-agent framework of logic of knowing how, which can handle the diversity of agents' abilities and formally specified plans based on each agent's own knowledge. 
%    \item A syntactic notion of agent-based epistemic plans where each agent can only use its own knowledge-that. 
    \item A PTIME model checking algorithm for the full language.  
    \item A complete axiomatization over finite models with perfect recall to ensure the semantics works as desired.\footnote{Axiomatizations can also help us to do abstract syntactic reasoning about know-how without fixing a model.}  
\end{itemize}

% Notion of consistency matters (whether they share the same set of actions?).
%     \begin{itemize}
%          \item Is it possible that $\Kh_2\phi\land \K_1\neg \Kh_2\phi$? 
%         \item Is it possible that $\Kh_2\phi\land \neg \K_1\neg \Kh_2\phi$?
%         \item Is it possible that $\Kh_2\phi\land \neg \Kh_1\Kh_2\phi$?
%         \item Is it possible that $\Kh_2\phi\land  \Kh_1\neg \Kh_2\phi$? 
%         \item Is it possible that $\neg \Kh_1\phi \land \K_1\Kh_2\phi$?
%         \item Is it possible that $\neg \Kh_1\phi \land \Kh_1\Kh_2\phi$?
%         \item Is it possible that $\neg \Kh_1\phi \land \Kh_1\Kh_1\phi$?
%         \item Is it possible that $\neg \Kh_1\phi \land \Kh_1\Kh_2\Kh_1\phi$
%     \end{itemize}

\subsection{Related work}
Doing automated planning using model checking techniques is not a new idea (cf. e.g., \cite{Giunchiglia2000}). The most distinctive feature of our framework, compared to other approaches to (epistemic) planning such as \cite{Bolander2015}, is that we can express the planning problem in the object language and incorporate higher-order epistemic planning. Following \cite{ghallab2004automated}, we base our framework on explicit-state models to lay out the theoretical foundation, instead of using more practical compact representations of models and actions which can generate the explicit-state models with certain properties (cf. \cite{LiW19,LiWangAIJ2021} for the discussions on the technical connections of the two). In contrast to the informal, model-dependent notions of plans such as policies, we use a plan specification language to formalize knowledge-based plans inspired by \cite{LangZ12,Lang2013,ZanuttiniLSS20}, but with more expressive branching conditions and more general constructors. Instead of the general structural operational semantics in \cite{LiWangAIJ2021}, we give a more intuitive semantics for our knowledge-based plans in this paper.   

Comparing to the coalition-based know-how approaches such as \cite{Naumov2018a}, we do not have any group know-how modality, but use a much more general notion of plans than the single-step joint action there. Besides the axiomatizations, which are the main focus in the previous work on the logic of know-how, we study the model checking complexity that is crucial for the application of our work to planning. 
A \textit{second-order know-how modality} $H_C^D$ was proposed in \cite{Naumov2018}, where $H_C^D\phi$ says the coalition $C$ knows how coalition $D$ can ensure $\phi$ by a single-step joint action, though $D$ may not know how themselves. 
%Note that it may happen that the coalition $D$ does not know how to ensure $\phi$ themselves due to their limited (distributed) knowledge. 
Note that $H_{\{i\}}^{\{j\}}\phi$ is neither $\Kh_i\Kh_j\phi$ nor $\K_i\Kh_j\phi$ in our setting, where $i$ may not know the plan that others may use to reach the goal. 

The $\Kh_i$ operator is clearly also related to the strategy-based \textit{de re} variants of the (single-agent) ATL operator $<\!\!< i >\!\!>$ under imperfect information (cf. e.g., \cite{Schobbens04,JamrogaA07}). In the setting of concurrent games of ATL, the strategies are functions from finite histories (or epistemic equivalence classes of them) to available actions, which induce the resulting plays of the given game that are usually infinite histories, on which temporal properties can be verified. This ``global'' notion of strategies leads to the problem of revocability of the strategies when  nesting the $<\!\!< i >\!\!>$ operators: can you change the strategy when evaluating $<\!\!< i >\!\!>\phi$ in the scope of another such operator \cite{Herzig15,AgotnesGJ07}? In contrast, it is crucial that our witness plans for know-how are not global and always terminate, and there is no controversy in the nested form of $\Kh_i$ operators in our setting. It is also crucial that the plans are executed sequentially in our setting, e.g., $\Kh_i\Kh_j\phi$ simply means agent $i$ knows how to make sure agent $j$ knows how to make sure $\phi$ \textit{after agent $i$ finishes executing the witness plan for its know-how}. Moreover,  while the \textit{de re} variants of ATL are useful in reasoning about (imperfect information) concurrent games (cf. e.g., \cite{MaubertPSS20,MaubertMPSS20}), our framework is closer to  the standard automated planning in AI over transition systems.  We leave it to the full version of the paper for a detailed technical comparison.     

\paragraph{Structure of the paper} In the rest of the paper, we first introduce briefly an epistemic language in Section \ref{sec.pre} to be used to specify the branching conditions of our plan specification language in Section \ref{sec.specl}. In Section \ref{sec.mc} we introduce the semantics of our know-how language and give a PTIME model checking algorithm. Finally, we obtain a complete axiomatization in Section \ref{sec.ax} and conclude with future directions in Section \ref{sec.conc}. 

\medskip
Due to the strict space limitation, we have to omit some proofs.  

% \begin{itemize}
% %    \item Why introduce know-how in planning besides the philosophical purpose?
%     \item Why using explicit models?
%     \item Why modal logic matters in planning?
%     \item Why syntactic langauage?
%     \item Complexity-wise for free..
% \end{itemize}

\section{Preliminaries}\label{sec.pre}
% To discuss all the later notions of know-how in a unified way, we will introduce a programming language to specify various types of plans. Some plans may involve branching structures by testing certain conditions. To specify those conditions precisely, we first introduce a simple logical language ($\Lan^{\EAL}_{\Act}$ below) and its semantics (cf. e.g., \cite{Wang2012}). Note that $\Lan^{\EAL}_{\Act}$ is only used to specify the conditions in the plans, and it  is \textit{not} the language of the logic of know-how, which will be introduced in the next section. 
%\subsection{A language for the branching conditions}
To specify the knowledge-based plans, we need the following formal language $\EAL$ to express the knowledge conditions. 
	\begin{definition}[\EAL\ Language]
	Let $\BP$ be a set of propositional letters, $\Ag$  be a set of agents, and $\BA$ be a set of atomic actions. The language %$\Lan^{\EAL}_{\Act}$ 
	of Epistemic Action Logic ($\EAL^\Act_\Ag$) is defined as follows: 
		$$\begin{array}{r@{\quad::= \quad}l}
		\varphi  &
		\top\mid p
		\mid \neg \varphi
		\mid (\varphi \land \varphi)
		\mid \K_i\phi 
		\mid [a]\phi\\
		%\mid \Halt{\pi}\\
% 		\pi &   \epsilon\mid \test{\phi}\mid a \mid (\pi\mpl\pi)\mid (\pi;\pi)\mid \pi^*\\
		%  \psi & \top\mid p\mid \neg \psi\mid (\psi\wedge \psi)\mid \K\psi
		\end{array}$$
		where $p\in\BP$, $i\in\Ag$,  %$q\in \BO$, $r\in\BP\cup\BO$, 
		and $a\in\BA$. 
	\end{definition}
The auxiliary connectives and modalities $\to, \lor, \lr{a}$ are defined as abbreviations as usual. 
%	In the later sections, various fragments of the \EPDL\ programs are used to specify linear and branching plans. 
	
	The semantics of \EAL\ formulas is defined on finite models with both epistemic relations and transitions labelled by atomic actions, which will also be used for the know-how logic. In contrast with the single-agent model of \cite{LiWangAIJ2021}, for each agent $i$ we have a special set of actions $\Act_i$ that are executable by $i$. Note that an action may be executable by multiple agents. Moreover, we implicitly assume each agent is aware of all the actions, even for those not executable by itself.  % (see Section \ref{sec.explicit-state} for the choice of using such models). 
	\begin{definition}[Model]\label{def.model}
		An Epistemic Transition System (i.e., a model) $\M$ is a tuple $$\lr{W,\{\sim_i\mid i\in\Ag\},\{\Act_i\mid i\in\Ag\},\{Q{(a)}\mid a\in\bigcup_{i\in\Ag} \Act_i \},V}$$ where:
		% \begin{itemize}
			% \item 
			$W$ is a non-empty set;
			% \item 
			$\Act_i$ is a set of actions for each $i\in\Ag$;
			% \item 
			$\sim_i \subseteq {W\times W}$ is an equivalence relation for each $i\in\Ag$;
			% \item 
			$Q{(a)}\ \subseteq {W\times W}$ is a binary relation for each $a\in \bigcup_{i\in\Ag} \Act_i$; 
			% \item 
			$V:W\to 2^\BP$ is a valuation.
		% \end{itemize}
			A frame is a model without the valuation function.
	\end{definition}
	\paragraph{Notations}
In the rest of the paper, \textbf{we use $\bigA$ to denote $\bigcup_{i\in\Ag}\Act_i$} for notational simplicity.
%if $\Ag$ is obvious from the context.
	We use $[s]^i$ to denote the set $\{t\in W\mid s\sim_i t\}$ and $[W]^i$ to denote the set $\{[s]^i\mid s\in W\}$. We sometimes call the equivalence class $[s]^i$ an $i$-belief-state. %, called an $i$-belief-state.
	% The binary relation $Q(a)$ on $W$  also can be seen as a function from $W$ to the power set of $W$, such that for each $s\in W$, $Q(a)(s)=\{t\in W\mid (s,t)\in Q(a)\}$. Thus, by abusing the notation, we sometimes also write $(s,t)\in Q(a)$ as $t\in Q(a)(s)$. 
    A model captures the agents' abilities to act and the uncertainty about the states. 
	%Note that in contrast to \cite{LangZ12,Naumov2018}, we do not assume that all the actions are executable on all the states.\footnote{Intuitively, not all the actions are executable under all the scenarios. Moreover, even as a technical assumption, assuming universal executability may make the unbounded linear plan never terminates, e.g., recall the plan 111\dots for opening a safe mentioned in the introduction.}
	
	${\EAL}^{\bigA}_\Ag$ formulas are given truth values on pointed epistemic transition systems.
	\begin{definition}[Semantics]
	Given a pointed model $(\M,s)$ where $\M=\lr{W,\{\sim_i\mid i\in\Ag\},\{\Act_i\mid i\in\Ag\},\{Q{(a)}\mid a\in \bigA %\bigcup_{i\in\Ag} \Act_i 
	\},V}$ and a formula $\phi\in{\EAL}^{\bigA}_\Ag$, the semantics is given below: 
	%satisfaction relation on $\phi$ and pointed model $(\M,s)$ is defined as follows:
		$$\begin{array}{|rll|}
		\hline
		\M,s\vDash \top & \Leftrightarrow & \textit{always}\\
		\M,s\vDash p & \Leftrightarrow & s\in \V(p)\\
       \M,s\vDash \neg\phi & \Leftrightarrow & \M,s\nvDash \phi\\
       \M,s\vDash (\phi\wedge\psi) &\LRA & \M,s\vDash\phi \text{ and } \M,s\vDash \psi\\
		\M,s\vDash \K_i\phi&\Leftrightarrow&\text{for all }t,\text{ if }s\sim_i t \text{ then } \M,t\vDash\phi\\
		\M,s\vDash [a]\phi&\Leftrightarrow& \text{for all }t,\text{ if }(s,t)\in Q(a) \text{ then } \M,t\vDash\phi\\
% 		\hline 
% 		\text{ where $Q(\pi)$ is (when $\pi\not=a)$}& &\\
% 		Q(?\phi)&:=&\{(s,s)\mid \M,s\vDash\phi \}\\
% 		Q(\pi_1\mpl\pi_2)&:=&Q(\pi_1)\cup Q(\pi_2)\\
% 		Q(\pi_1;\pi_2)&:=&Q(\pi_1)\circ Q(\pi_2)\\
% 		Q(\pi^*)&:=&\bigcup_{k\in\omega}Q(\pi^k)\\
		\hline 
		\end{array}$$
	 A formula is valid on a frame if it is true on all the pointed models based on that frame. 
	\end{definition}
	
Let $X$ be a set of states. We sometimes use $\M,X\vDash\phi$ to denote that $\M,s\vDash\phi$ for all $s\in X$. Thus $\M,X\nvDash\phi$ denotes that $\M,s\nvDash\phi$ for some $s\in X$, e.g., we may write $\M,[s]^i\vDash\K_i\phi$ or $\M,[s]^i\nvDash\K_i\phi$. 

	Here are some standard results about bisimulation that will be useful in the later technical discussion (cf.\ e.g., \cite{mlbook}).
	\begin{definition}[Bisimulation]\label{def.bis}
	Given a model $\M=\lr{W,\{\sim_i\mid i\in\Ag\},\{\Act_i\mid i\in\Ag\},\{Q{(a)}\mid a\in\bigcup_{i\in\Ag} \Act_i \},V}$,
	a {non-empty} \emph{symmetric} binary relation $Z$ on the domain of  $\M$ is called a {bisimulation} iff whenever $(w,v)\in Z$ the following conditions are satisfied:
\begin{itemize}
\item[(1).] For any $p\in\BP :p\in V(w)\iff p\in V(v)$.
\item[(2).] %For the bianry relation $\sim$:
    %\begin{itemize}
        %\item 
        for any $i\in\Ag$, if $w\sim_i w'$ for some $w'\in W$ then there exists a $v'\in W$ such that $v\sim_i v'$ and $w'Zv'$.
%\item if $v\sim  v'$ then there exists a $w'$ such that $w\sim w'$ and $w'Zv'$.
 %   \end{itemize}
\item[(3).] for any $a\in \bigcup\Act_i$,
    %\begin{itemize}
        %\item 
        if $(w, w')\in Q(a)$ for some $w'\in W$  then there exists a $v'\in W$ such that $(v,v')\in Q(a)$ and $w'Zv'$.
% \item if $(v, v')\in Q(a)$ then there exists a $w'$ such that $(w,w')\in Q(a)$ and $w'Zv'$.
%     \end{itemize}
\end{itemize} 

The pointed models $(\M,w)$ and $(\M,v)$ are said to be \textit{bisimilar} ($\M,w\bis\M,v$) if there is a bisimulation $Z$ such that $(w,v)\in Z.$ 
\end{definition}

We use  $\M,[s]^i\bis\M,[s']^i$ to denote that for each $t'\in[s']^i$,  there exists $t\in [s]^i$ such that $\M,t\bis\M,t'$ and vice versa. 
Given two pointed models $(\M,s)$ and $(\M,s')$,  we use $\M,s\equiv_{{\EAL}_\Ag^{\bigA}}\M,s'$ to denote that for each $\phi\in {\EAL}_\Ag^{\bigA}$, $\M,s\vDash\phi$ if and only if $\M,s'\vDash\phi$. Similarly, $\M,s\equiv_{\EAL_\Ag^{\bigA}|_{\K_i}}\M,s'$ denotes that for each $\K_i\phi\in \EAL_\Ag^{\bigA}$, $\M,s\vDash\K_i\phi$ if and only if $\M,s'\vDash\K_i\phi$.

Please note that $\M,s\bis\M,s'$ implies $\M,[s]^i\bis\M,[s']^i$,  but the other way around does not work in general. 
	
% 	\begin{proposition}\label{pro.BisImplEquiv}
% 		If $\M,s\bis\M,s'$, we have that $\M,s\equiv_{\Lan^{\EAL}_{\Act}}\M,s'$.
% 	\end{proposition}
As an adaption of the Hennessy-Milner theorem for modal logic (cf. e.g., \cite{mlbook}), we have:
	\begin{proposition}\label{pro.FiniteEquivImplBis} \label{pro.BisImplEquiv}
		Given a finite model $\M$, for each state $s$ in $\M$, we have that $\M,s\equiv_{\EAL_\Ag^{\bigA}}\M,s'$ iff  $\M,s\bis\M,s'$.
	\end{proposition}
	Moreover, if we only look at the $\K_i\phi$ formulas, we have: 
	\begin{proposition}\label{pro.CorresBetweenEClassAndEFormula}
		Given a finite model $\M$, for each state $s$ in $\M$, we have that $\M,[s]^i\bis \M,[s']^i$ iff $ \M,s\equiv_{\EAL_\Ag^{\bigA}|_{\K_i}} \M,s' $.
	\end{proposition}
% 	\begin{proof}
% 		\textbf{Left to Right}: Assuming that $ \M,s\not\equiv_{\EAL_\Ag^{\bigA}|_{\K_i}} \M,s' $, it follows that there is $\K_i\phi$ such that $\M,s\vDash\K_i\phi$ but $\M,s'\nvDash\K_i\phi$. It follows by $\M,[s]^i\bis \M,[s']^i$ that there is some $u$ such that $s\sim_i u$ and $\M,u\bis\M,s'$. By $s\sim_i u$, we have that $\M,u\vDash\K_i\phi$. By $\M,u\bis\M,s'$ and Proposition~\ref{pro.BisImplEquiv}, we know that $\M,s'\vDash\K_i\phi$. This is contradictory to  $\M,s'\nvDash\K_i\phi$. Thus, if  $\M,[s]^i\bis \M,[s']^i$ then $ \M,s\equiv_{\EAL_\Ag^{\bigA}|_{\K_i}} \M,s' $.
		
% 		\textbf{Right to Left}: %Assume that $\M,[s]\not\bis \M,[s']$.
% 		Firstly, we will show that for each $t\in[s']^i$ there is some $u\in[s]^i$ such that $\M,u\bis\M,t$. If not, it follows that there is some $t\in [s']^i$ such that $\M,t\not\bis\M,u$ for each $u\in [s]^i$. By Proposition~\ref{pro.FiniteEquivImplBis}, we have that for each $u\in[s]^i$, there is some formula $\phi_u$ such that $\M,t\nvDash\phi_u$ but $\M,u\vDash\phi_u$. Since $\M$ is finite,  the disjunction of all these $\phi_u$ is a formula, denoted $\psi$. From this follows that $\M,s\vDash \K_i\psi$ and $\M,t\vDash\neg\psi$. Since $t\sim_i s'$, $\M,s'\vDash\neg\K_i\psi$. This is contradictory to  $\M,s\equiv_{\EAL_\Ag^{\bigA}|_{\K_i}}\M,s'$. Thus, we have shown that for each $t\in[s']^i$ there is some $u\in[s]^i$ such that $\M,u\bis\M,t$. From a similar process, we can show that for each $u\in[s]^i$ there is some $t\in[s']^i$ such that $\M,u\bis\M,t$. Therefore, $\M,[s]^i\bis\M,[s']^i$.
% 	\end{proof}
Most approaches of epistemic planning implicitly assume the property of perfect recall (cf. \cite[Sec. 6.5]{LiWangAIJ2021}), here we will also consider this extra property which will play a role in model checking. 	

\begin{definition}[Perfect recall]
	Given a model $\M$, for each $i\in\Ag$ and each $a\in\Act_i$,
	if $(s_1,s_2)\in Q(a)$ and $s_2\sim_i s_4$ then there must exist a state $s_3$ such that $s_1\sim_i s_3$ and $(s_3,s_4)\in Q(a)$. 
\end{definition}
In words, if the agent cannot distinguish two states after doing $a$ then it could not distinguish them before. 
Perfect recall (\texttt{PR}) corresponds to the following property depicted below: %(where dotted lines link states in the same uncertainty set): 
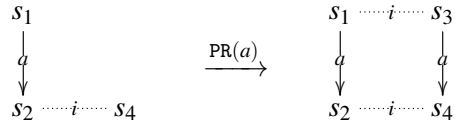
\begin{center}	
	\begin{minipage}{0.20\textwidth}	$$\xymatrix{{s_1}\ar[d]|a&\\
			s_2\ar@{..}[r]|i&s_4
			% \save "2,1"."2,2"!C="g1"*+[F--:<+20pt>]\frm{}
			% \restore
	 }$$
		%\caption{$\M,s_1$}
	\end{minipage} 
	\begin{minipage}{0.05\textwidth}
		$$\xrightarrow{\ZIG{a}}$$
	\end{minipage}
	\begin{minipage}{0.20\textwidth}	$$\xymatrix{{s_1}\ar[d]|a\ar@{..}[r]|i&s_3\ar[d]|a\\
			s_2\ar@{..}[r]|i&s_4
			% \save "1,1"."1,2"!C="g1"*+[F--:<+20pt>]\frm{}
			% \restore
			% \save "2,1"."2,2"!C="g1"*+[F--:<+20pt>]\frm{}
			% \restore
	 }$$
		%\caption{$\M,s_1$}
	\end{minipage} 
\end{center}
%In this paper, we will focus on models with perfect recall.

\section{A specification language for knowledge-based plans} \label{sec.specl}	

Inspired by \cite{LangZ12}, we only consider knowledge-based plans for each agent. We introduce the following specification language which is a multi-agent variant of a fragment of the programming language introduced in \cite{LiWangAIJ2021}.
%a knowledge-based programming language as the specification language of plans. %To stay general, we do not restrict ourselves to knowledge-based programs here.  
\begin{definition}[Knowledge-based specification language $\Prg(i)$]
% Given a set \Act\ of primitive actions, 
Based on $\EAL^{\bigA}_\Ag$, $\Prg(i)$ is defined as follows:
	$$\begin{array}{r@{\quad::= \quad}l}
		\pi &   \epsilon \mid a \mid  (\pi;\pi)\mid \ifthel{\K_i\phi}{\pi}{\pi} %\mid \whlo{\phi}{\pi}\\
		%  \psi & \top\mid p\mid \neg \psi\mid (\psi\wedge \psi)\mid \K\psi
		\end{array}$$
where {$\epsilon$ is the empty plan, $a\in \Act_i$ and $\phi\in{\EAL}^{\bigA}_\Ag$}.
% \begin{itemize}
%     \item the empty plan $\epsilon$ is a \Prg,
%     \item any action $a\in\Act$ is a \Prg,
%     \item if $\pi$ and $\pi'$ are \Prg s, then $\pi;\pi'$ is a \Prg,
%     \item for \Prg s $\pi$, $\pi'$ and $\phi\in\Lan^{\EAL}_{\Act}$, $\ifthel{\phi}{\pi}{\pi'}$ is a \Prg,
%     \item for a \Prg\ $\pi$ and $\phi\in\Lan^{\EAL}_{\Act}$, $\whlo{\phi}{\pi}$ is a \Prg.
% \end{itemize}
\end{definition}
Note that the atomic action $a$ in the above definition must belong to $i$, but the condition $\K_i\phi$ can contain epistemic modalities and actions of other agents, e.g., $\K_i\neg\K_j[b]p$ is a legitimate condition for a knowledge-based plan for $i$ even when $b\not\in \Act_i$. 

We will abbreviate $\ifthel{\K_i\phi}{\pi_1}{\pi_2}$ as $\shortIf{\K_i\phi}{\pi_1}{\pi_2}$.
We use $\pi^n$ to denote the program $\underbrace{\pi;\cdots;\pi}_{n}$, and $\pi^0$ is the empty program $\epsilon$.

Let $\M$ be an epistemic transition system, $s$ be a state in $\M$, and $\pi$ be a program in $\Prg(i)$. The state set $Q(\pi)(s)$, which is the set of states on which executing $\pi$ on $s$ will terminate, is defined by induction on $\pi$ as follows:
\[
	\begin{aligned}
		Q(\epsilon)(s) &= \{s\}\\
		Q(a)(s)&=\{t\mid (s,t)\in Q(a)\}\\
		Q(\pi_1;\pi_2)(s) &=\{v\mid\text{there is }t\in Q(\pi_1)(s): v\in Q(\pi_2)(t)\}\\
		Q(\shortIf{\K_i\phi}{\pi_1}{\pi_2})(s)&= 
		\begin{cases}
			Q(\pi_1)(s) & \M,s\vDash\K_i\phi\\
			Q(\pi_2)(s) & \M,s\nvDash\K_i\phi
		\end{cases}
	\end{aligned}
\]

We use $Q(\pi)([s]^i)$ to denote the set $\bigcup_{s'\in[s]^i}Q(\pi)(s')$.
% We use $\pi^n$ to denote the program $\underbrace{\pi;\cdots;\pi}_{n}$ and $\pi^0$ is the empty plan $\epsilon$.
Given the property of perfect recall, the resulting states of a knowledge-based plan have a nice property.  
\begin{proposition}\label{pro.terminationClosedOverSimI}
	If $\M$ has perfect recall and $\pi\in\Prg(i)$, then $Q(\pi)([s]^i)$ is closed over $\sim_i$.
\end{proposition}
Next we will define a notion of strong executability which is a generalized version of the one in \cite{Wang2016}. 
\begin{definition}[Strong executability]
	We define a program $\pi$ to be \emph{strongly executable} on a pointed model $(\M,s)$ by induction on $\pi$ as follows:
	% \begin{itemize}
		% \item 
		(1)$\epsilon$ is always strongly executable on $\M,s$;
		% \item 
		(2)$a$ is strongly executable on $\M,s$ if $Q(a)(s)\neq\emptyset$;
		% \item 
		(3)$\pi_1;\pi_2$ is strongly executable on $\M,s$ if $\pi_1$ is strongly executable on $\M,s$ and $\pi_2$ is strongly executable on each $t\in Q(\pi_1)(s)$;
		% \item 
		(4)$\shortIf{\K_i\phi}{\pi_1}{\pi_2}$ is strongly executable  on $\M,s$ if either $\M,s\vDash\K_i\phi$ and $\pi_1$ is strongly executable on $\M,s$, or $\M,s\nvDash\K_i\phi$ and $\pi_2$ is strongly executable on $\M,s$.
	% \end{itemize}
	We say that $\pi$ is strongly executable on a set $X$ of states if it is strongly executable on $\M,s$ for each $s\in X$.
\end{definition}
To define the model checking algorithm later, we need to construct the tree of possible executions of a plan where each node is essentially an epistemic equivalence class (usually called a \textit{belief state} in automated planning). 
\begin{definition}[Execution tree]\label{def.executionTree}
	Given a model $\M$ and an agent $i\in\Ag$, %$\M=\lr{W,\{\sim_i\mid i\in\Ag\},\{\Act_i\mid i\in\Ag\},\{Q{(a)}\mid a\in\bigA \},V}$, 
	an execution tree of $\M$ for $i$ is a labeled tree $\T=\lr{N,E,L}$, where $\lr{N,E}$ is a tree consisting of a nonempty set $N$ of nodes and a set $E$ of edges and $L$ is a label function that labels each node in $N$ with an $i$-belief-state and each edge in $E$ with an action $a\in\Act_i$, such that, for all $(n,m)\in E$,
	% \begin{itemize}
		% \item 
		(1)$L(n,m)$ is strongly executable on $L(n)$; %there are $s\in L(n)$ and $t\in L(m)$ such that $(s,t)\in Q(L(n,m))$;
		% \item 
		(2)$L(n,m)=L(n',m')$ if $(n',m')\in E$ and $n=n'$;
		% \item 
		(3)there exists $t\in L(m)$ such that $t\in Q(L(n,m))(L(n))$;
		% \item  
		(4)for each $t\in Q(L(n,m))(L(n))$, there exists $k\in N$ such that $(n,k)\in E$ and $t\in L(k)$. % and $L(n,k)=L(n,m)$. 
	% \end{itemize}
	% We say that $T$ is deterministic if all edges starting from the same node are labeled the same action, i.e., if $n=n'$ then $L(n,m)=L(n',m')$.
\end{definition}

% Given an execution tree, 
We use $r^\T$ (or simply $r$) to refer to the root node of the execution tree $\T$.
The following two propositions show the relations between execution trees and knowledge-based plans.

% We say that an execution tree of $\M$, $T=\lr{D,E,L}$, is induced by a program $\pi\in\Prg(i)$, if the following hold:
% \begin{itemize}
% 	\item if $\pi$ is $\epsilon$, then $D=\{r\}, E=\emptyset$, and $L(r)\in [W]^i$. % for some state $s$.
% 	\item if $\pi$ is some $a\in\bigcup\Act_i$, then $L(r)\in [W]^i$, and for each node $r\neq n\in D$, we have that $(r,n)\in E$, $L(r,n)=a$, and $L(n)=[t]^i$ for some $t\in Q(a)([s]^i)$.
% 	\item if $\pi$ is $\pi_1;\pi_2$, then there is $D'\subseteq D$ such that $T|_{D'}$ is an execution tree induced by $\pi_1$ and for each terminal node $n$ of $T|_{D'}$, the subtree of $T$ generated by $n$ (denoted by $T|_n$) is an execution tree of $\pi_2$ and $T=T|_{D'}\cup (\bigcup\{T|_n\mid n$ is a terminal node of $T|_{D'}\})$.
% 	\item if $\pi$ is $\shortIf{\K_i\phi}{\pi_1}{\pi_2}$, then either $\M,s\vDash\K_i\phi$ and $T$ is an execution tree induced by $\pi_1$, or $\M,s\nvDash\K_i\phi$ and $T$ is an execution tree induced by $\pi_2$, where $s$ is some state in $L(r)$.
% \end{itemize}

\begin{proposition}\label{pro.fromProgramToExecutionTree}
	Given a model $\M$ and a program $\pi\in\Prg(i)$, if $\pi$ is strongly executable on $[s]^i$, then there is a finite execution tree $\lr{N,E,L}$ of $\M$ for $i$ such that %$T$ is  induced by $\pi$.
	% \begin{itemize}
		% \item  
		$L(r)=[s]^i$ and that
		% \item 
		for each leaf node $k$, $L(k)\subseteq Q(\pi)([s]^i)$.
	% \end{itemize}
\end{proposition}
\begin{proposition}\label{pro.fromExecutionTreeToProgram}
	Let $\T=\lr{N,E,L}$ be a finite execution tree of $\M$ for $i$. % be an agent in $\Ag$. 
	% If $L(n)\in [W]^i$ for all $n\in D$ and $L(k,m)\in\Act_i$ for each $(k,m)\in E$, then 
	There exists $\pi\in\Prg(i)$ such that 
	% \begin{itemize}
		% \item 
		$\pi$ is strongly executable on $L(r)$ and that
		% \item 
		for each $t\in Q(\pi)(L(r))$, there exists a leaf node $k$ such that 
		$\M,[t]^i\bis\M,L(k)$.
	% \end{itemize}  
\end{proposition}

\section{Model checking knowledge-how}\label{sec.mc}
Extending the work of \cite{LiWangAIJ2021}, we introduce the multi-agent epistemic language of know-how and its semantics and study the model checking complexity in detail. 
\begin{definition}[\ELKh\ Language]
	Let $\BP$ be a set of variables and $\Ag$ be a set of agents, the language  of epistemic logic with the  know-how operator (\ELKh) is defined as follows:
		\[\phi::=\top\mid p\mid\neg\phi\mid (\phi\land\phi)\mid\K_i\phi\mid \Kh_i\phi \]
		where $p\in\BP$ and $i\in\Ag$.
\end{definition}

\begin{definition}[Semantics]\label{def.sem} The truth conditions for basic propositions,  Boolean connectives and the epistemic operator $\K_i$ are as before in the case of \EAL. 
% We give a generic form of ten different semantics $\vDash_x$ for the know-how operator $\Kh$ where $x$ is one of $$\texttt{PLAN=\{S, L, P, T, U, C, K, F, UK, CK}\}.$$ 
		%The first eight letters stand for simple, linear, skippable, stoppable, unbounded, conditional, knowledge-based, and full plans respectively.
		%the last two stand for unbounded and conditional knowledge-based plans.  
		%{$\texttt{UK}$ stands for unbounded knowledge-based plans, and $\texttt{CK}$ stands for conditional knowledge-based plans, i.e. knowledge-based plans without while-loop.}
		
		$$\begin{array}{|ll|}
		   \hline \M,s\vDash\Kh_i\phi\iff&\text{there is a plan }\pi\in\Prg(i)\text{ such that}\\
		&1.\ \pi\text{ is strongly executable on } [s]^i;\\
		&2.\ \M,t\vDash\phi \textsf{ for each }t\in Q(\pi)([s]^i).\\
		\hline
		\end{array}$$
\end{definition}

The readers can go back to Example \ref{ex.symptom} to verify the truth of the formula mentioned there.  

We can show that $\ELKh$ is invariant under bisimualtion defined in Definition \ref{def.bis} which also match labelled transitions (cf. \cite{LiWangAIJ2021}).
\begin{proposition}\label{pro.BisImplKnowHowEquiv}
	If $\M,s\bis\M,u$ then $\M,s\vDash\phi$ iff $\M,u\vDash\phi$ for all $\phi\in\ELKh$.
\end{proposition}

In the remainder of this section, we will propose a model checking algorithm for $\ELKh$. 
The key part of the model checking algorithm is to check the $\Kh_i$-formulas, and the problem of whether $\M,s\vDash\Kh_i\phi$ depends on whether there is a good plan for $\Kh_i\phi$. 
Our strategy is to reduce the problem of whether there is a good plan for $\Kh_i\phi$ to the nondeterministic planning problem in \cite{Cimatti2003}, which we will briefly introduce below.

\begin{definition}[Planning problem \cite{Cimatti2003}]
	Let $T=\lr{D,\Act,\{\rel{a}\mid a\in\Act\}}$ be a labeled transition system, where $D$ is a finite set, $\Act$ is a set of actions, and $\rel{a}\subseteq D\times D$ is a binary relation.
	A fully observable nondeterministic planning problem is a tuple $P=\lr{T,s,G}$ where $s\in D$ is the initial state, and $G\subseteq D$ is the goal set.
	% \begin{itemize}
	% 	\item 
	% 	\item $s\in D$ is the initial state;
	% 	\item $G\subseteq D$ is the goal set.
	% \end{itemize} 
\end{definition}
Given a labeled transition system $T$ (note that there is no epistemic relation), we often write $(s,t)\in\rel{a}$ as $s\rel{a}t$. The sequence $s_0\rel{a_1}s_1\rel{a_2}\cdots s_n$ is called an execution trace of $T$ from $s_0$. % if $(s_k,s_{k+1})\in \rel{a_{k+1}}$ for all $0\leq k<n$.
A partial function $f:D\to \Act$ is called a \textit{policy}. 
An execution trace $s_0\rel{a_1}s_1\rel{a_2}\cdots s_n$ is  induced by $f$ if $f(s_k)=a_{k+1}$ for all $0\leq k<n$. It is complete w.r.t. $f$ if $s_n$ is not in the domain of $f$.

\begin{definition}[Strong plan]
	A policy $f$ is a strong plan for a planning problem $P=\lr{T,s,G}$ if each complete execution trace induced by $f$ from $s$ is finite and terminates in $G$.
\end{definition}

\begin{proposition}\label{pro.PlanningInPtime}
	Let $\T=\lr{D,\Act,\{\rel{a}\mid a\in\Act\}}$.
	The existence of strong plans for a planning problem $P=\lr{T,s,G}$ is in \Ptime\ in the size of $T$.
\end{proposition}
\begin{proof}
It is shown in \cite{Cimatti2003} that the procedure presented in Algorithm \ref{algo.planning} always terminates and is correct for whether $P$ has a strong plan.
Moreover, the loop is executed at most $|D|+1$ times. Therefore, the procedure is in \Ptime.
\end{proof}

\begin{algorithm}[t]%\rm
	\SetKw{Not}{not}
	\SetKw{And}{and}
	\SetKw{Guess}{guess}
	\SetKwFunction{Mc}{MC}
	\SetKwFunction{Cr}{CR}
	\SetKwFunction{StrongPE}{StrongPlanExistence}
	\SetKwProg{Prg}{Procedure}{:}{}
	\SetAlgoVlined
	\Prg%(\tcc*[f]{ check whether $\M,w\vDashk\phi$})
	{\StrongPE($T,s, G$)}{
		$OldSA :=$ false\;
		$SA :=\emptyset$\;
		\While(\tcc*[f]{$STs(SA)$ is the state set $\{v\in D\mid (v,a)\in SA$ for some $a\in\BA\}$.}){$OldSA\neq SA$ \And $s\not\in G\cup STs(SA)$}{
			$PreI := \{(v,a)\in D\times\BA\mid v\not\in G\cup STs(SA), v$ has $a$-successors, and all $a$-successors of $v$ are in $G\cup STs(SA)\}$\;  %\emptyset\neq\rel{a}([s])\subseteq (G\cup STs(SA))$\; %\tcc*{Let }
			%  $ there exists $[t]$ such that $[s] \rel{a}[t]$, and if $[s]\rel{a}[v]$ then $[v]\in G\cup STs(SA)  \}$\;
			% $NewSA := \{([s],a)\in PreI\mid [s]\not\in G\cup Sts(SA)\}$\;
			$OldSA := SA$\;
			$SA := SA\cup PreI$\;
		}
		\lIf{$s\in G\cup STs(SA)$}{\Return true}
		\Return false\;
	}
		\caption{A planning procedure}\label{algo.planning}
\end{algorithm} 

%{\color{red}
%Secondly, %to reduce the existence of a good plan for $\Kh_i\phi$ in a model $\M$ to the existence of strong plan, 
Now we define the nondeterministic planning problem that corresponds to $\M,s\vDash\Kh_i\phi$.
%based on a model $\M$. 
Since the language of $\ELKh$ is a multi-agent one where different agents have different abilities, the corresponding nondeterministic planning problem must take the agent $i$ into account. This makes the reduction more complex than that of single-agent know-how logics.
%}

Given a model $\M=\lr{W,\{\sim_i\mid i\in\Ag\},\{\Act_i\mid i\in\Ag\},\{Q{(a)}\mid a\in\bigA \},V}$ and an agent $i$,
a planning problem for $i$ is $P(\M,i)=\lr{T(\M,i),[w]^i,G}$ where $[w]^i$ is an $i$-belief-state, $G$ is a set of $i$-belief-states, and $T(\M,i)=\lr{D,\Act_i,\{\rel{a}\mid a\in\Act_i\}}$ is a labeled transition system where 
% \begin{itemize}
	% \item 
	$D=[W]^i$ and
	% \item 
	for each $a\in\Act_i$, $[s]^i\rel{a}[t]^i\iff a$ is strongly executable on $[s]^i$, and there exists $t'\in[t]^i:t'\in Q(a)([s]^i)  \}$.
% \end{itemize}

We will show that the problem of whether $\M,s\vDash\Kh_i\phi$ indeed can be reduced to a nondeterministic planning problem for $i$ (Lemma \ref{lemma.equivalenceBetweenKnowHowAndPlanning}). Before that, we need two auxiliary propositions.

\begin{proposition}\label{pro.programImplyPlanningHasSolution}
	Given a model $\M$ and an agent $i$, if a program $\pi\in\Prg(i)$ is strongly executable on $[s]^i$, then the planning problem $P=\lr{T(\M,i),[s]^i,G}$ where $G=\{[v]^i\mid \M,[v]^i\bis\M,[t]^i$ for some $ t\in Q(\pi)([s]^i)\}$  has a strong plan.
\end{proposition}

\begin{proposition}\label{pro.planningHasSolutionImplyProgram}
	Given  $\M$ and $i$, if a planning problem for $i$, $P(\M,i)=\lr{T(\M,i),[s]^i,G}$, has a strong plan, then there exists $\pi\in\Prg(i)$ such that 
	% \begin{itemize}
		% \item 
		$\pi$ is strongly executable on $[s]^i$ and that
		% \item 
		for each $t\in Q(a)([s]^i)$, there exists $[v]^i\in G$ such that
		 $\M,t\bis\M,v$.
	% \end{itemize}
\end{proposition}

\begin{lemma}\label{lemma.equivalenceBetweenKnowHowAndPlanning}
	Given a model $\M$,
	we have that
	$\M,s\vDash\Kh_i\phi$ iff the planning problem $P(\M,i)=\lr{T(\M,i),[s]^i,G}$ where $G=\{[t]^i\mid \M,[t]^i\vDash\K_i\phi\}$ has a strong plan.
\end{lemma}
\begin{proof}
	If $\M,s\vDash\Kh_i\phi$, it follows that there exists $\pi\in\Prg(i)$ such that $\pi$ is strongly executable on $[s]^i$ and that $\M,t\vDash\phi$ for each $t\in Q(\pi)([s]^i)$. 
	By Proposition \ref{pro.terminationClosedOverSimI}, it follows that $\M,[t]^i\vDash\K_i\phi$ for each $t\in Q(\pi)([s]^i)$.
	Let $G'$ be the set $\{[v]^i\mid \M,[v]^i\bis\M,[t]^i$ for some $t\in Q(\pi)([s]^i)\}$. By Proposition~\ref{pro.BisImplKnowHowEquiv}, we have that $\M,[v]^i\vDash\K_i\phi$ for each $[v]^i\in G'$. It follows that $G'\subseteq G$. 
	Let $P'$ be the planning problem $\lr{T(\M,i),[s]^i,G'}$. By Proposition~\ref{pro.programImplyPlanningHasSolution}, the planning problem $P'$ has a strong plan. Since $G'\subseteq G$, it follows that the planning problem $P(\M,i)$ als has a strong plan.

	If the planning problem $P(\M,i)$ has a strong plan, by Proposition \ref{pro.planningHasSolutionImplyProgram}, it follows that there exists $\pi\in\Prg(i)$ such that $\pi$ is strongly executable on $[s]^i$ and that for each $t'\in Q(a)([s]^i)$, there exists $[t]^i\in G$ such that $\M,t'\bis\M,t$. Since $\M,t\vDash\phi$ for each $[t]^i\in G$, by Proposition \ref{pro.BisImplKnowHowEquiv}, it follows that $\M,t'\vDash\phi$ for each $t'\in Q(a)([s]^i)$. We then have that $\M,s\vDash\Kh_i\phi$. 
\end{proof}

Finally, we are ready to show the upper bound of model checking $\ELKh$.

\begin{theorem}\label{theo.complexityOfMC}
	The model checking problem of \ELKh\ is in \Ptime\ in the size of models.
\end{theorem}
\begin{proof}
	Algorithm \ref{algo.modelChecking} presents the model checking algorithm \texttt{MC} for $\ELKh$ on epistemic transition systems with perfect recall.
	By Lemma \ref{lemma.equivalenceBetweenKnowHowAndPlanning}, we call the procedure $\texttt{StrongPlanExistence}$ presented in Algorithm \ref{algo.planning} to check whether $\M,w\vDash\Kh_i\phi$. 
	Proposition \ref{pro.PlanningInPtime} shows that the procedure $\texttt{StrongPlanExistence}$ can be computed in \ptime\ of the size of $T(\M,i)$. Since the size of $T^\M$ is bounded by the size of $\M$, it follows that Algorithm \ref{algo.planning} is in \ptime\ in the size of $\M$. 
\end{proof}

\begin{algorithm}[t]%\rm
	\SetKw{Not}{not}
	\SetKw{And}{and}
	\SetKw{Guess}{guess}
	\SetKwFunction{Mc}{MC}
	\SetKwFunction{Cr}{CR}
	\SetKwFunction{StrongPE}{StrongPlanExistence}
	\SetKwProg{Prg}{Procedure}{:}{}
	\SetAlgoVlined
	\Prg(\tcc*[f]{whether $\M,w\vDash\phi$}){\Mc{$\M,w,\phi$}}{
		\Switch(\tcc*[f]{Boolean cases are omitted}){$\phi$}{    
			
			% \uCase{$p$}
			% {
			% 	\eIf{$p \in V(s)$}
			% 	{
			% 		\Return true\;
			% 	}
			% 	{
			% 		\Return false\;
			% 	}
			% }
			% %\uCase{$\top$}
			% %{
			% %	\Return true\
			% %}
			% \uCase{$\neg \phi$}
			% {	\Return \Not{MC$(\M, (e,s), \phi)$}
			% 	%\eIf{MC$(\calM, s, \sigma, \phi) = $ false}
			% 	%{
			% 	%	\Return true\
			% 	%}
			% 	%{
			% 	%	\Return false\
			% 	%}
			% }
			% \uCase{$\phi_1 \land \phi_2$}
			% {
			% 	\Return MC$(\M, (e,s),\phi_1)$ \And MC$(\M, (e,s), \phi_2)$\;
			% 	%\eIf{MC$(\calM, s, \sigma, \phi_1) = $ true  {\bf and} MC$(\calM, s, \sigma, \phi_2) = $ true}
			% 	%{
			% 	%	\Return true\
			% 	%}
			% 	%{
			% 	%	\Return false\
			% 	%}
			% }
			\uCase{$\K_i \phi$}
			{
				\ForEach{$w'\in [w]^i$}
				{
%					\ForEach{$s'\in L(e')$}
%					{
						\lIf{\Not \Mc{$\M,w',\phi$}}{
							\Return false
						}
%					}
				}
				\Return true\;				
			}
			\Case{$\Kh_i\phi$}
			{
				$G:=\emptyset$\;
				\ForEach{$[s]^i$ in $T(\M,i)$}{
					\lIf{\Mc($\M,s,\K_i\phi$)}{ $G:=G\cup\{[s]^i\}$ }
				}
				\Return \StrongPE~($T(\M,i),[w]^i, G$)\;
			}
		}
	}
\caption{A model checking procedure}\label{algo.modelChecking}
\end{algorithm}

\section{Axiomatization}\label{sec.ax}
In this section, we axiomatize our logic of know-how over finite models with perfect recall. Note that we cannot apply the generic completeness result in \cite{LiWangAIJ2021} since our setting is not only multi-agent but also with an extra condition of perfect recall. It turns out the axiomatization is the multi-agent variant of the system in \cite{Fervari2017} which has an extra axiom $\AxKhtoKhK$ than the proof system in \cite{LiWangAIJ2021}. Note that although the proof system is similar, the semantics of the $\Kh$ in \cite{Fervari2017} is based on model-dependent functions from equivalence classes to actions instead of the syntactically specified plan in our work. 

The proof system $\SLKhc$ is presented in Table \ref{tab.axiomatics}.
%\begin{center}
	\begin{table}[htbp]
		\begin{minipage}[t]{0.5\textwidth}
	\begin{tabular*}{0.55\textwidth}{lc}
		%		\multicolumn{2}{c}{System $\SKh$}\\
		% \lcline{1-2}\rcline{3-4}
		{\textbf{Axioms}}&\\
		%\lcline{1-2}\rcline{3-4}
		\TAUT & \text{all axioms of propositional logic}\\
		\DISTK & $\K_i p\land\K_i (p\to q)\to \K_i q$\\
		\AxTrK& $\K_i p\to p $ \\
		\AxTransK& $\K_i p\to\K_i\K_i p$\\
		\AxEucK& $\neg \K_i p\to\K_i\neg\K_i p$\\
		\AxKtoKh &$\K_i p \to \Kh_i p$ \\	
		\AxKhtoKhK&$\Kh_i p \to \Kh_i\K_i p$  \\
		\AxKhtoKKh&$\Kh_i p \to \K_i\Kh_i p$  \\	
		%\AxKhEucK&$\neg\Khx p \to \K\neg\Khx p$ && \phantom{$\dfrac{\phi}{\Box\phi}$} \\
		\AxKhKh&$\Kh_i \Kh_i p \to \Kh_i p$  \\
		\AxKhbot&$\neg \Kh_i \bot$  \\
	\end{tabular*}
\end{minipage}
\begin{minipage}[t]{0.35\textwidth}
	\begin{tabular*}{0.35\textwidth}{lclc}
		\textbf{Rules}\\
		\MP & $\dfrac{\varphi,\varphi\to\psi}{\psi}$&\NECK &$\dfrac{\varphi}{\K_i\varphi}$\\
		\EQREPKh& $\dfrac{\varphi\to\psi}{\Kh_i\varphi\to\Kh_i\psi}$ & \SUB & $\dfrac{\varphi}{\varphi[\psi\slash p]}$
	\end{tabular*}
\end{minipage}
	\caption{$\SLKhc$}\label{tab.axiomatics}
\end{table}
%\end{center}

% \noteYW{We sould highlight the differences bewtween the proof here and proof in \cite{LiWangAIJ2021}. We may omit some proofs in this subsection. Maybe we should also explain why don't we need the more general composition axiom as in \cite{LiWangAIJ2021}.}

% Compared to the proof system presented in \cite{LiWangAIJ2021}, we have an extra axiom $\AxKhtoKhK$ in $\SLKhc$, because we focus on models with perfect recall in this paper. 
The property of perfect recall is sufficient to guarantee the validity of $\AxKhtoKhK$, but the axiom cannot characterize the property, as it is shown in \cite{LiWangAIJ2021}.
Moreover, the composition axiom $\AxKhKh$ in $\SLKhc$ is a weaker version of the axiom $\Kh(\phi\lor\Kh\psi)\to\Kh(\phi\lor\psi)$ in \cite{LiWangAIJ2021}.

Due to the axiom $\AxKhtoKKh$, the stronger composition axiom in \cite{LiWangAIJ2021} can be derived in $\SLKhc$ based on $\AxKhKh$.
\begin{proposition}
	$\vdash \Kh_i(\phi\lor\Kh_i\psi)\to\Kh_i(\phi\lor\psi)$.
\end{proposition} 
% \begin{proof}
% By \AxKhtoKhK, it follows that $(1)\vdash\Kh_i(\phi\lor\Kh_i\psi)\to \Kh_i\K_i(\phi\lor\Kh_i\psi)$.
% Since $\vdash \Kh_i\psi\lra\K_i\Kh_i\psi$, with axioms and rules about $\K_i$, it can be shown that $\vdash \K_i(\phi\lor\Kh_i\psi)\to \K_i\phi\lor\Kh_i\psi$. 
% With \AxKtoKh, we then have that $\vdash \K_i(\phi\lor\Kh_i\psi)\to \Kh_i\phi\lor\Kh_i\psi$. 
% With \EQREPKh, we then have that  $\vdash \Kh_i\K_i(\phi\lor\Kh_i\psi)\to \Kh_i(\Kh_i\phi\lor\Kh_i\psi)$.
% With (1), we then have that $(2)\vdash \Kh_i(\phi\lor\Kh_i\psi)\to \Kh_i(\Kh_i\phi\lor\Kh_i\psi)$.

% Since $\vdash\phi\to(\phi\lor\psi)$, by \EQREPKh, it follows that $\vdash \Kh_i\phi\to\Kh_i(\phi\lor\psi)$.
% Similarly, we have that $\vdash \Kh_i\psi\to\Kh_i(\phi\lor\psi)$. These imply that $\vdash \Kh_i\phi\lor\Kh_i\psi\to\Kh_i(\phi\lor\psi)$. 
% By \EQREPKh,, we then have that $\vdash \Kh_i(\Kh_i\phi\lor\Kh_i\psi)\to\Kh_i\Kh_i(\phi\lor\psi)$.
% With (2),  it follows that $\vdash \Kh_i(\phi\lor\Kh_i\psi)\to \Kh_i\Kh_i(\phi\lor\psi)$.
% With \AxKhKh, we then have that $\vdash \Kh_i(\phi\lor\Kh_i\psi)\to \Kh_i(\phi\lor\psi)$.
% \end{proof}
%{\color{red}

Next, we will show that the proof system $\SLKhc$ is sound and complete over finite models with perfect recall.
By Propositions \ref{pro.fromProgramToExecutionTree} and \ref{pro.fromExecutionTreeToProgram}, it follows that there is a correspondence between execution trees and plans. With this correspondence result, we can  significantly simplify the proofs of the soundness and completeness of the proof system $\SLKhc$ than \cite{Fervari2017} and \cite{LiWangAIJ2021}. 
More specifically, the validity of Axiom \AxKhKh, which was highly non-trivial in \cite{Fervari2017}, follows from the fact that two subsequent execution trees can be combined into one execution tree. (The detailed proof can be found in the following theorem of soundness.)
For the completeness, the corresponding execution tree plays an important role in the proof of the truth lemma (Lemma \ref{lem.truth}).
%}

\begin{theorem}[Soundness]
	The proof system $\SLKhc$ is sound over finite models with perfect recall. % with respect to\ $\vDashx$ for all $x\in \allindex.$
	\end{theorem}
\begin{proof}
The validity of axioms $\AxTrK, \AxTransK, \AxEucK$ is due to the standard semantics for $\K_i$. 
%For the validity of other axioms, we just need to use the two properties in Proposition  \ref{prop.khcondition} shared by all the $\vDashx$ semantics. 

The axiom $\AxKtoKh$ says if $p$ is known then you know how to achieve $p$. Its validity is guaranteed the empty program. 
The axiom $\AxKhtoKhK$ is valid due to Proposition \ref{pro.terminationClosedOverSimI}.
The axiom $\AxKhtoKKh$ is the positive introspection axiom for $\Kh_i$ whose validity is due to the semantics for $\Kh_i$. 
The axiom $\AxKhbot$ says we cannot guarantee contradiction, which indirectly requires that the plan should terminate at someplace and is guaranteed by the condition (1) of the semantics for $\Kh_i$. 

Finally, we will show the validity of the axiom $\AxKhKh$, that is, $\M,s\vDash\Kh_i\phi$ if $\M,s\vDash\Kh_i\Kh_i\phi$.
If $\M,s\vDash\Kh_i\Kh_i\phi$, it follows that there is a program $\pi\in\Prg(i)$ such that $\pi$ is strongly executable on $[s]^i$ and that $\M,t\vDash\Kh_i\phi$ for each $t\in Q(\pi)([s]^i)$.
By Proposition \ref{pro.fromProgramToExecutionTree}, there is a finite execution tree $\T$ of $\M$ for $i$ such that $L(r)=[s]^i$ and that for each leaf node $k$, $L(k)\subseteq Q(\pi)([s]^i)$.
It follows that $\M,L(k)\vDash\Kh_i\phi$ for each leaf node $k$. We then have that for each leaf node $k$, there is a program $\pi_k\in\Prg(i)$ such that $\pi_k$ is strongly executable on $L(k)$ and that $\M,v\vDash\phi$ for all $v\in Q(\pi_k)(L(k))$.
By Proposition \ref{pro.fromProgramToExecutionTree} again, for each leaf $k$ of $\T$, there is a finite execution tree $\T_k$ of $\M$ for $i$ such that $L^{\T_k}(r^{\T_k})=L(k)$ and that for each leaf node $l$ of $\T_k$, $L^{\T_k}(l)\subseteq Q(\pi_k)(L(k))$.
Moreover, since $\M,v\vDash\phi$ for all $v\in Q(\pi_k)(L(k))$, it follows that $\M,L^{\T_k}(l)\vDash\K_i\phi$ for each leaf node $l$ of $\T_k$.
We then construct an execution tree $\T'$ of $\M$ for $i$ by extending each leaf node $k$ of $\T$ with $\T_k$.
Since $\T$ and all $\T_k$ are finite, it follows that $\T'$ is finite.
Moreover, we have that the root of $\T'$ is labeled with $[s]^i$ and $\M,L^{\T'}(l)\vDash\K_i\phi$ for each leaf node $l$ of $\T'$.
By Proposition \ref{pro.fromExecutionTreeToProgram}, there is a program $\pi'\in\Prg(i)$ such that $\pi'$ is strongly executable on $[s]^i$ and that for each $v\in Q(\pi')([s]^i)$, there exists a leaf node $l$ of $\T'$ such that $\M,[v]^i\bis\M,L^{\T'}(l)$. Since $\M,L^{\T'}(l)\vDash\K_i\phi$ for each leaf node $l$ of $\T'$, by Proposition \ref{pro.BisImplKnowHowEquiv}, we have that $\M,v\vDash\K_i\phi$ for all $v\in Q(\pi')([s]^i)$. Thus, we have that $\M,s\vDash\Kh_i\phi$.
\end{proof}

Next, we will show the completeness of $\SLKhc$.
The key is to show that every consistent formula is satisfiable. We will construct a canonical model that consists of all atoms which are maximal consistent sets with respect to a closure set of formulas and show that every consistent formula is satisfied in the canonical model.

Given a set of formulas $\Delta$, let:
$\Delta|_{\K_i}=\{\K_i\phi\mid\K_i\phi\in\Delta \}$, 
$\Delta|_{\neg\K_i}=\{\neg\K_i\phi\mid\neg\K_i\phi\in\Delta \}$, 
$\Delta|_{\Kh_i}=\{\Kh_i\phi\mid\Kh_i\phi\in\Delta \}$, 
$\Delta|_{\neg\Kh_i}=\{\neg\Kh_i\phi\mid\neg\Kh_i\phi\in\Delta \}$.
Let $\SFC$ be a subformula-closed set of $\ELKh$-formulas that is finite. %The closure $\cl{\SFC}$ is $\SFC\cup\{\K\phi,\neg\K\phi\mid\phi\in\SFC\}$.  
	%It is obvious that $\SFC$ is finite. %countable since the whole language itself is countable. 
% The set $\cl{\SFC}$ is the closure $\SFC\cup\{\K_i\phi,\neg\K_i\phi\mid\phi\in\SFC,i\in\Ag$ occurs in $\SFC\}$. 
\begin{definition}\label{def.closure}
	The closure $\cl{\SFC}$ is defined as:
	\begin{center}
		$\SFC\cup\{\K_i\phi,\neg\K_i\phi\mid\phi\in\SFC,i\in\Ag$ occurs in $\SFC\}$. 
	\end{center}	 
\end{definition}
	If $\SFC$ is finite, so is $\cl{\SFC}$.
	Next, we will use it to build a canonical model with respect to\ $\Phi$.
\begin{definition}[Atom]\label{def.atom}
	We enumerate the formulas in $\cl{\SFC}$ by $\{\psi_0,\cdots, \psi_h\}$ where $h\in \mathbb{N}$. The formula set $\Delta=\{Y_k\mid k\leq h\}$ is an atom of $\cl{\SFC}$ if 
	\begin{itemize}
		\item $Y_k=\psi_k$ or $Y_k=\neg\psi_k$ for each  $\psi_i\in \cl{\SFC}$;
		\item $\Delta$ is consistent in $\SLKhc$.
	\end{itemize}
\end{definition}

The following two propositions are similar with their single-agent versions in \cite{Fervari2017}. 
%We present the detailed proofs in Appendix \ref{proof.pro.exeistenceLemmaForK} and \ref{proof.pro.AtomShareTheSameKh} respectively.

\begin{proposition}\label{pro.exeistenceLemmaForK}
Let $\Delta$ be an atom of $\cl{\SFC}$ and $\K_i\phi\in\cl{\SFC}$. If $\K_i\phi\not\in \Delta$ then there exists $\Delta'$ such that $\Delta'|_{\K_i}=\Delta|_{\K_i}$ and $\neg\phi\in\Delta'$.
\end{proposition}

\begin{proposition}\label{pro.AtomShareTheSameKh}
	Let $\Delta$ and $\Delta'$ be two atoms of $\cl{\SFC}$ such that $\Delta|_{\K_i}=\Delta'|_{\K_i}$. We have $\Delta|_{\Kh_i}=\Delta'|_{\Kh_i}$.% for all $\Delta'\in[\Delta]$.
\end{proposition}
% \begin{proof}%%see appendix
% 	%It is obvious by the Axioms \AxKhtoKKh\ and \AxTrK\ and the definition of $\sim_c$.
% 	For each $\Kh_i\phi\in\Delta$, by Definition \ref{def.closure}, we have that $\Kh_i\phi\in\SFC$. It follows that $\K_i\Kh_i\phi\in\cl{\SFC}$.
% 	For each $\Kh_i\phi\in\Delta$, by the Axiom \AxKhtoKKh , we have $\K_i\Kh_i\phi\in\Delta$. Since $\Delta|_{\K_i}=\Delta'|_{\K_i}$, it follows that $\K_i\Kh_i\phi\in\Delta'$. By Axiom \AxTrK, we then have that $\Kh_i\phi\in\Delta'$. Thus, we have showed that $\Kh_i\phi\in\Delta$ implies $\Kh_i\phi\in\Delta'$. Similarly we can prove that $\Kh_i\phi\in\Delta'$ implies $\Kh_i\phi\in\Delta$. Hence, $\Delta|_{\Kh_i}=\Delta'|_{\Kh_i}$.% for all $\Delta'\in[\Delta]$.
% \end{proof}

% 	Let $\Phi$ be a subformula-closed set of formulas.
% 	%As above, let 
% 	%$\SFC^\BA=\{(\phi,1),(\phi,-1)\mid\Khcx\phi\in\SFC \}$, 
% 	Let $\SFC^\BA=\{\phi\in\Lanc\mid \Khc\phi\in\SFC \}$, 
% 	and let $g:\BA\to\SFC^\BA$ be a surjective function.
\begin{definition}[Canonical model]\label{def.canonicalKHC}	
	Given a subformula-closed set $\SFC$, the canonical model $\M^{\SFC}=\lr{W,\{\sim_i\mid i\in\Ag\},\{\Act_i\mid i\in\Ag\},\{Q{(a)}\mid a\in\bigcup_{i\in\Ag} \Act_i \},V}$ is defined as:	
	\begin{itemize}
		\item $W=\{\Delta\mid\Delta$ is an atom of $\cl{\SFC}\}$;
		\item for each $i\in\Ag$, $	\Delta\sim_i \Gamma\iff\Delta|_{\K_i}=\Gamma|_{\K_i}$;
		\item for each $i\in\Ag$, $\Act_i=\{(i,\phi)\mid \Kh_i\phi\in\SFC\}$;
		\item for each $(i,\phi)\in\bigcup\Act_i$, we have that $	(\Delta,\Gamma)\in Q(\phi)\iff \Kh_i\phi\in \Delta$ and $\K_i\phi\in\Gamma$;
		
		\item for each $p\in\SFC$, $p\in V(\Delta)\iff p\in\Delta$.
	\end{itemize}
\end{definition}

\begin{proposition}\label{pro.canonicalModelIsPerfectRecall}
	If $\Gamma\in Q(i,\phi)(\Delta)$ and $\Gamma'\in [\Gamma]^i$, then $\Gamma'\in Q(i,\phi)([\Delta]^i)$.
	In other words, the model $\M^\SFC$ has perfect recall.
\end{proposition}
\begin{proof}
	We only need to show that $\Gamma'\in Q(i,\phi)(\Delta)$. Since $\Gamma\in Q(i,\phi)(\Delta)$, it follows that $\Kh_i\phi\in\Delta$ and $\K_i\phi\in\Gamma$. Since $\Gamma'\in [\Gamma]^i$, it follows that $\K_i\phi\in\Gamma'$. Thus, we have that $\Gamma'\in Q(i,\phi)(\Delta)$.
\end{proof}

\begin{proposition}\label{prop.AtomModelAllSuccessorHaskh}
	Let $\Delta$ be a state in $\M^{\SFC}$ and $(i,\psi)\in\ACT_i$ be executable at $\Delta$. If $\Kh_i\phi\in\Delta'$ for all $\Delta'\in Q(i,\psi)(\Delta)$ then $\Kh_i\phi\in\Delta$.
\end{proposition}
\begin{proof}
% 	We present the detailed proof in Appendix \ref{proof.prop.AtomModelAllSuccessorHaskh}
	First, we show that $\K_i\psi$ is not consistent with $\neg\Kh_i \phi$. %Assume that $\K\psi$ is consistent with $\neg\Kh \phi$.
	Since $(i,\psi)\in\ACT_i$ is executable at $\Delta$, it follows that $\Kh_i\psi\in\Delta$ and that there is some state in $Q(i,\psi)(\Delta)$ that contains $\K_i\psi$.
	Moreover, it is obvious that $\Kh_i\phi\in\cl{\SFC}$.  
	Assume that $\K_i\psi$ is consistent with $\neg\Kh_i \phi$. 
	By a Lindenbaum argument, there exists an atom $\Gamma$ of $\cl{\SFC}$ such that $\{\K_i\psi,\neg\Kh_i\phi \}\subseteq \Gamma$. 
	By the definition of $\M^\SFC$, it follows that $(\Delta,\Gamma)\in Q(i,\psi)$.
	Since we know that $\Kh_i\phi\in\Delta'$ for all $\Delta'\in Q(i,\psi)(\Delta)$, it follows that $\Kh_i\phi\in\Gamma$.
	It is contradictory with the fact that $\Gamma$ is consistent. 
	Thus, $\K_i\psi$ is not consistent with $\neg\Kh_i \phi$. Hence, we have that $\vdash \K_i\psi\to \Kh_i\phi$. 
	
	Since $\vdash \K_i\psi\to \Kh_i\phi$, it follows by Rule \EQREPKh\ and Axiom \AxKhtoKhK\ that $\vdash\Kh_i\psi\to \Kh_i\Kh_i\phi$. Moreover, it follows by Axiom \AxKhKh\ that $\vdash\Kh_i\psi\to \Kh_i\phi$. 
	Since we have shown that  $\Kh_i{\psi}\in \Delta$, we have that $\Kh_i\phi\in\Delta$.% by the definition 
\end{proof}
%Now we are ready to obtain the truth lemma. 
\begin{lemma}[Truth Lemma]\label{lem.truth}
	For each $\phi\in\cl{\SFC}$,  $\M^{\SFC},\Delta\vDash\phi$ iff $\phi\in\Delta$.
\end{lemma}
\begin{proof}
	We prove it by induction on $\phi$. 
	The atomic and Boolean cases are straightforward. The case of $\K_i\phi$ can be proved by  Proposition~\ref{pro.exeistenceLemmaForK}.
		Next, we only focus on the case of $\Kh_i\phi$.

	\textbf{Right to Left}: If $\Kh_i\phi\in\Delta$, we will show $\M^{\SFC},\Delta\vDash\Kh_i\phi$.  
	% Firstly, there are two cases: $\K\phi\in \Delta$ or $\K\phi\not\in\Delta$. If $\K\phi\in\Delta$, then $\K\phi,\phi\in\Delta'$ for all $\Delta'\in[\Delta]$. Since $\phi\in\SFC$, it follows by IH that $\M^{\SFC},\Delta'\vDash\phi$ for all $\Delta'\in[\Delta]$. Therefore, $\M^{\SFC},\Delta\vDash\K\phi$. It follows by Axiom \AxKtoKh\ and the soundness of $\SKh$ that $\M^{\SFC},\Delta\vDash\Kh\phi$. 
	We first show that $\K_i\phi$ is consistent. If not, namely $\vdash\K_i\phi\to\bot$, it follows by Rule \EQREPKh\ that $\vdash \Kh_i\K_i\phi\to\Kh_i\bot$. It follows by Axiom \AxKhbot\ that $\vdash\Kh_i\K_i\phi\to\bot$. Since $\Kh_i\phi\in\Delta$, it follows by Axiom \AxKhtoKhK\ that $\Delta\vdash\bot$, which is in contradiction with the fact that $\Delta$ is consistent. Therefore, $\K_i\phi$ is consistent. 
		
		By Lindenbaum's Lemma, %Proposition \ref{pro.lindenbaumForAtom}
		there exists an atom $\Gamma$ such that $\K_i\phi\in\Gamma$. 
		By the definition of $\M^\SFC$, it follows that $(i,\phi)\in\Act_i$ and that $(\Delta',\Gamma)\in Q(i,\phi)$ for each $\Delta'\in [\Delta]^i$. It means $(i,\phi)$ is strongly executable on $[\Delta]^i$. 
		For each $\Theta\in Q(i,\phi)([\Delta]^i)$, by the definition of $\M^\SFC$, it follows that $\K_i\phi\in\Theta$. By Axiom \AxTr, it follows that $\phi\in\Theta$. By IH, we then have that $\M^\SFC,\Theta\vDash\phi$ for each $\Theta\in Q(i,\phi)([\Delta]^i)$. Thus, we have that $\M^\SFC,\Delta\vDash\Kh_i\phi$.
		
		\textbf{Left to Right}: Suppose $\M^{\SFC},\Delta\vDash\Kh_i\phi$, we will show $\Kh_i\phi\in\Delta$.
		Since $\M^{\SFC},\Delta\vDash\Kh_i\phi$, it follows that there exists $\pi\in\Prg(i)$ such that $\pi$ is strongly executable on $[\Delta]^i$ and that $\M^\SFC,\Gamma\vDash\phi$ for all $\Gamma\in Q(\pi)([\Delta]^i)$.
		By IH, we have that $\phi\in\Gamma$ for all $\Gamma\in Q(\pi)([\Delta]^i)$.
		Moreover, for each $\Gamma\in Q(\pi)([\Delta]^i)$, if $\Gamma'\in [\Gamma]^i$, by Proposition \ref{pro.canonicalModelIsPerfectRecall}, it follows that $\Gamma'\in Q(\pi)([\Delta]^i)$, and we then have that $\phi\in\Gamma'$.
		Moreover, with Proposition \ref{pro.exeistenceLemmaForK}, it is easy to check that $\K_i\phi\in\Gamma$ for all $\Gamma\in Q(\pi)([\Delta]^i)$.

		Since $\pi$ is strongly executable on $[\Delta]^i$, by Proposition \ref{pro.fromProgramToExecutionTree}, there is a finite execution tree $\T$ of $\M^\SFC$ for $i$ such that $L(r)=[\Delta]^i$ and that $L(k)\subseteq Q(\pi)([\Delta]^i)$. Since we have shown that $\K_i\phi\in\Gamma$ for all $\Gamma\in Q(\pi)([\Delta]^i)$, it follows that if $k$ is a leaf node then $\K_i\phi$ is in all states in $L(k)$.
		% Next, we will show that $\Kh_i\phi\in\Delta$. 
		Next, we firstly show the following claim:
		\[\text{for each node }k\text{ and each }\Theta\in L(k): \Kh_i\phi\in\Theta.\]
		Please note that the execution tree $\T$ is finite.
		We prove the claim above by induction on the height of nodes.
		If the height of $k$ is $0$, it means that $k$ is a leaf node. It follows that $\K_i\phi\in\Theta$ for each $\Theta\in L(k)$. By Axiom \AxKtoKh, it follows that $\Kh_i\phi\in\Theta$ for each $\Theta\in L(k)$.
		With the IH that the claim holds for each node whose height is less than $h$, we will show that the claim holds for nodes whose height is $h\geq 1$. Given a node $k$ whose height is $h$, since $h\geq 1$, it follows that there is a node $k'$ such that $(k,k')$ is an edge in $\T$ and the height of $k'$ is less than $h$, in other words, $k'$ is a child node of $k$. Let $L(k)=[\Theta]^i$ and $L(k,k')=(i,\psi)$. Since $\T$ is an execution tree, it follows that $(i,\psi)$ is strongly executable on $[\Theta]^i$, and then $(i,\psi)$ is executable on $\Theta$. For each $\Theta'\in Q(i,\psi)(\Theta)\subseteq Q(L(k,k'))(L(k))$, by the definition of execution trees, it follows that there is a $k$-child $k''$ such that $\Theta'\in L(k'')$. Since $k''$ is a child of $k$, by IH, it follows that $\Kh_i\phi\in \Theta'$. Thus, we have shown that $\Kh_i\phi\in \Theta'$ for each $\Theta'\in Q(i,\psi)(\Theta)$. By Proposition \ref{prop.AtomModelAllSuccessorHaskh}, it follows that $\Kh_i\phi\in\Theta$.
		Thus, we have shown the claim.
		Since the root is labeled with $[\Delta]^i$, we then have that $\Kh_i\phi\in\Delta$.
\end{proof}

Let $\phi$ be a consistent formula.
By Proposition \ref{pro.exeistenceLemmaForK}, it follows that there is an atom $\Delta$ of $\cl{\SFC} $ such that $\phi\in\Delta$, where $\SFC$ is the set of subformulas of $\phi$. By Proposition \ref{pro.canonicalModelIsPerfectRecall} and Lemma \ref{lem.truth}, it follows that $\phi$ is satisfiable over finite models with perfect recall. The completeness then follows:
\begin{theorem}[Completeness]
	The proof system $\SLKhc$ is complete over finite models with perfect recall.	
\end{theorem}

From the construction of the canonical model $\M^\SFC$ in Definition \ref{def.canonicalKHC}, we can see that both the state set $W$ and the action set $A$ are bounded. This implies that the logic has a small model property. Moreover,  by Theorem \ref{theo.complexityOfMC},
the decidability then follows: 

\begin{theorem}
	The logic \ELKh\ is decidable.
\end{theorem}
%\noteYW{Decidability?}

\section{Conclusion} \label{sec.conc}
In this paper, we propose the notion of higher-order epistemic planning by using an epistemic logic of knowing how. The planning problems can be encoded using model checking problems in the framework, which can be computed in PTIME in the size of the the model. We also axiomatize the logic over finite models with perfect recall.

As for future work, besides many theoretical questions about the knowing how logic as it is, we may extend its expressive power to capture conditional knowledge-how, which is very useful in reasoning about planning problems. For example, one may say I know how to achieve $\phi$ given $\psi$. Note that it is very different from $\Kh_i(\psi\to \phi)$ or $\K_i\psi\to \Kh_i\phi$. The important difference is that such conditional knowledge-how is global, compared to the current local semantics of $\Kh_i$. This is similar to the binary global $\Kh_i$ operator introduced in \cite{Wang2016}. It would be interesting to combine the two notations of know-how in the same framework. On the practical side, we can consider the model checking problems over compact representations of the actions and states using the techniques in \cite{LiW19,LiWangAIJ2021} connecting the explicit-state models with the compact representations.

% One drawback of our current work is that we cannot express precisely: once I know how to do this then I know how to do that.  
% \begin{itemize}
% \item $\Kh_1 (\phi\to \psi)$ does not express this... 
% \item $\K_1(\Kh_1\phi\to \Kh_2\psi)$ does not express this... 
% \item $\K_1\Box(\Kh_1\phi\to\Kh_1\psi)$ sounds good where $\Box$ is the transitive closure. 

% Decidable...
% \end{itemize}

% \begin{acks} %% 致谢
% % To Robert, for the bagels and explaining CMYK and color spaces.
% We thank the anonymous reviewers.
% \end{acks}

\bibliographystyle{eptcs}
\bibliography{sgwyjKH}

\end{document}
\endinput
%%
%% End of file `sample-sigconf.tex'.